\documentclass[11pt,notitlepage,onecolumn]{article}

\usepackage[utf8]{inputenc}
\usepackage{url}            % simple URL typesetting
\usepackage{booktabs}       % professional-quality tables
\usepackage{amsfonts}       % blackboard math symbols
\usepackage{nicefrac}       % compact symbols for 1/2, etc.
\usepackage{microtype}      % microtypography  % LM: This package breaks url package! 
\usepackage{graphicx} % support the \includegraphics command and options
\usepackage{amsmath}
\usepackage{amsthm}
\usepackage{amssymb}
\usepackage[noend]{algorithmic}
\usepackage[ruled,vlined]{algorithm2e}

\newtheorem{theorem}{Theorem}[section]
\newtheorem{lemma}[theorem]{Lemma} 

\newtheorem{proposition}[theorem]{Proposition}

\newtheorem{definition}[theorem]{Definition}
\newtheorem{conjecture}[theorem]{Conjecture}
\newtheorem{problem}[theorem]{Problem}

\newcommand{\lam}{\lambda}
\newcommand{\eps}{\varepsilon}

\DeclareMathOperator{\dist}{dist}
\RequirePackage[normalem]{ulem} %DIF PREAMBLE
\RequirePackage{color}\definecolor{RED}{rgb}{1,0,0}\definecolor{BLUE}{rgb}{0,0,1} %DIF PREAMBLE
\providecommand{\DIFaddbegin}{} %DIF PREAMBLE
\providecommand{\DIFaddend}{} %DIF PREAMBLE
\providecommand{\DIFdelbegin}{} %DIF PREAMBLE
\providecommand{\DIFdelend}{} %DIF PREAMBLE
\begin{document}
\title{Deep Learning and Hierarchical Generative Models} % hierarchical
\author{ Elchanan Mossel\DIFdelbegin %DIFDELCMD < \thanks{Supported by ONR grant N00014-16-1-2227   and 
%DIFDELCMD < NSF grants CCF-1665252 and DMS-1737944 {\tt elmos@mit.edu} } %%%
\DIFdelend \DIFaddbegin \thanks{Supported by ONR grant N00014-16-1-2227   and 
NSF CCF-1665252 and DMS-1737944 {\tt elmos@mit.edu} } \DIFaddend \\ MIT }

%\address{EM: Department of Statistics, UC Berkeley, Berkeley CA}
%\thanks{}

\date{\today}
\maketitle

\begin{abstract}
It is argued that  deep learning is efficient for data that is generated from hierarchal generative models. 
Examples of such generative models include wavelet scattering networks, functions of compositional structure, and deep rendering models. 
Unfortunately so far, for all such models, it is either not rigorously known that they can be learned efficiently, or it is not known 
that ``deep algorithms" are required in order to learn them. 

We propose a simple family of ``generative hierarchal models" which can be efficiently learned and where ``deep" algorithm are necessary for learning. Our definition of ``deep" algorithms is based on the empirical observation that deep nets necessarily use correlations between features. More formally, we show that in a semi-supervised setting,  given access to low-order moments of the labeled data and all of the unlabeled data, it is information theoretically impossible to perform classification while at the same time there is an efficient algorithm, that given all labelled and unlabeled data, perfectly labels all unlabelled data with high probability. 

For the proof, we use and strengthen the fact that Belief Propagation does not admit a good approximation in terms of linear functions. 

\end{abstract}

\maketitle

%{\bf : 
%\begin{itemize}
%\item Maybe reverse picture is right? Would explain transfer learning? i.e., deep part of tree 
%corresponds to lower levels of networks?
%\item Variant with better decoding properties - functions are not 2-2 but rather 2-4?
%With parameter sharing learning is also deep net?
%\item Continuous applications?
%\end{itemize}
%}

\newpage

\section{Introduction}
We assume that the reader is familiar with the basic concepts and developments in deep learning. 
We do not attempt to summarize the big body of work studying neural networks and deep learning. 
We refer readers who are unfamiliar with the area to~\cite{GoBeCo:16} and the references within. 

We hypothesize that deep learning is efficient in learning data that is generated from generative hierarchical models.
This hypothesis is in the same spirit of the work of Bruna, Mallat and others who suggested wavelet scattering 
networks~\cite{bruna2013invariant} as the generative model, the work of Mhaskar, Liao and Poggio who suggested compositional functions as the generative model~\cite{mhaskar2016learning} and work by Patel, Nguyen,  and Baraniuk~\cite{patel2015probabilistic} who suggested hierarchical rending models. 
Unfortunately so far, for all previous models, it is either not rigorously known that they can be learned efficiently, or it is not known 
that ``deep algorithms" are required in order to learn them. 

The approach presented in this paper is motivated by connections between deep learning and evolution.
In particular, we focus on simple generative evolutionary models as our generative processes. 
While these models are {\em not} appropriate models for images or language inference, they provide several advantages: 
\begin{itemize}
\item There are well established biological processes of evolution. Thus the generative model studied is not a human-made abstraction but the actual process that generated the data. 
\item The models are mathematically simple enough so we can provide very accurate answers as to the advantage of the depth in a semi-supervised setting.
\end{itemize} 

We further note that some of the most successful applications of deep learning are in labeling objects that are generated in an evolutionary fashion 
such as the identification of animals and breeds from images, see e.g.~\cite{LeBeHi:15} and the reference within. 
%We hypothesize that deep learning is efficient in learning data that is generated from generative hierarchical models.
%In particular we believe that natural languages and natural images are very well approximated by (yet to be discovered) generative hierarchical models. Such models process high-level abstract concepts via iterative refinement to lower level concepts ending with low-level representations. 
%Languages also evolve and the success of deep learning in natural language processing tasks also highlights the connection between evolution and deep learning. A final obvious connection is that deep networks are models of neural networks which evolve.  
Let us consider the problem of identifying species from images. One form of this problem was tackled by Darwin. In his {\em Evolution of Species}, Darwin used phylogenetic trees to summarize the evolution of species~\cite{Darwin:59}. The evolutionary tree, in turn, helps in identifying observed species. 
The problem of species identification took a new twist in the DNA age, where morphological characters of species were replaced by DNA sequences as the data for inference of the relationship between species~\cite{Felsenstein:04,SempleSteel:03}. 

%The problem of reconstruction of phylogenetic trees from DNA sequences is a central one in molecular biology, see e.g.\cite{Felsenstein:04,SempleSteel:03}. By now, there is a well-developed theory of evolutionary models, as well as a comprehensive theory of how well such models can be reconstructed both information theoretically and algorithmically see e.g.~\cite{ErStSzWa:99a,Mossel:03,Mossel:04a} and follow up work. 
%Thus, phylogenetic reconstruction provides a very interesting example of a generative hierarchical model as:
%\begin{itemize}
%\item The basic generative model is known. 
%\item Good theoretical understanding exists as to which algorithms are effective in reconstructing phylogenetic trees.
%\end{itemize}
Our hypothesis that deep learning is effective in recovering generative hierarchal models leads us to explore other generative models and algorithms to recover these models. While the models we will develop in the current work are restrictive, they represent an 
 attempt to extend the phylogenetic theory from the problem of reconstructing trees based on DNA sequences to reconstructing relationships based on different types of representations with the ultimate goal of understanding ``real" representations such  as 
 representations of natural languages and natural images. 
%This paper should be viewed as a first theoretical attempt in this direction. We hope that follow up work will strengthen the theory, extend the models and experiment with them.

In what follows we introduce a family of models. We start with the phylogenetic model. The phylogenetic model we use, the symmetric Markov model, is a classical model. However, we study it from a new perspective: 

\begin{itemize}
\item First - in addition to a DNA sequence, each node of the tree is associated with a label where different nodes might have the same label. For example, a node with a specific DNA sequence might have the label,''dog" or ''mammal". 
\item Second - we are interested in the semi-supervised learning problem, where the labels of a small subset of the data are known and the goal is to recover the labels of the remaining data.
\end{itemize}

We then define novel generative models which have additional features:

\begin{itemize}
\item Change of representation. In phylogenetic models, the representation is given by the DNA sequences (or RNA, proteins etc.), while it seems like in many deep learning situations, there isn't necessarily a canonical representation. We model this by introducing a permutation on the alphabet between every node and each of its descendants. 
\item Interaction between features. In classical phylogenetic models each letter evolves independently, while in most deep learning scenarios, the interaction between features is key. We introduce a model that captures this property. 
\end{itemize}

In order to establish the power of deep learning, we define two types of limited algorithms (which are not ''deep"):
\begin{itemize}
\item {\em Local algorithms}. Such algorithms have to determine the label of each data point  based on the labeled data only. The notion of local algorithms is closely related to the notion of supervised learning. Note, however, that local algorithms do not output a classifier after observing the labeled data ; instead for each sample of unlabeled data we run the local algorithm. 
\item {\em Shallow Algorithms}. These algorithms only use summary statistics for the labeled data. In other words, such algorithms are not allowed to utilize high order correlations between different features of the labeled data (our results will apply to algorithms that can use bounded order correlations). 
\end{itemize}

In our main results, we provide statistical lower bounds on the performance of local and shallow algorithms. We also provide efficient algorithms that are neither shallow nor local. Thus  our results provide a formal interpretation of the power of deep learning. 
In the conclusion, we discuss a number of research directions and open problems. 

\subsection{Related Work}
Our work builds on work in theoretical phylogenetics, with the aim of providing a new theoretical perspective on deep learning.
An important feature of our generative models is that they include both representation and labels. 
In contrast, most of the work in the deep learning literature focuses on the encoding within the deep network.  
Much of the recent work in deep learning deals with the encoding of data from one level to the next. In our model, we avoid this (important) aspect by considering only representations that are essentially 1 to 1 (if we exclude the effect of the noise or ``nuisance variables"). Thus our main focus is in obtaining rigorous results relating multi-level hierarchical models and classes of semi-supervised learning algorithms whose goal is to label data generated from the models.  

A main theme of research in the theory of deep networks is studying the expressive power of bounded width networks 
in terms of their depth, see e.g.~\cite{CoShSa:16,EldanShamir:16,Telgarsky:16} and also 
\cite{mhaskar2016learning}.
Our results show that learning of deep nets cannot be performed by simple methods that apply to shallow generative models. We also note that positive theoretical results of~\cite{ABGM:14}. 
These, however, are not accompanied by lower bounds showing that deep algorithms are needed in their setup.

\section{Hierarchical Generative Models}
In this section, we will define the generative models that will be discussed in the paper. 

\subsection{The space of all objects - a tree}
All the models will be defined on a $d$-ary tree $T=(V,E)$ of $h$ levels, rooted at 
$v_0$.

The assumption that the tree is regular is made for simplicity.
Most of the results can be extended to many other models of trees, including random trees. 

\subsection{Representations}

The representation is a function $R : V \to [q]^k$. The representation of node $v \in V$ is given 
by $R(v)$. 

In some examples, representations of nodes at different levels are of the same type.
For example, if we think of $T$ as a phylogenetic tree, then $R(v)$ may represent the DNA sequence of $v$. In other examples, $R(v)$ has a different meaning for nodes at different levels.
For example, if $T$ represents a corpus of images of animals, then $R(v)$ for a deep node $v$ may define the species and the type of background, while $R(v)$ at a lower level may represent the pixels of an image of an animal (this is an illustration - none of the models presented at the paper are appropriate for image classification). 

%In some illustrative examples 
%$R(v)$ may be be the objects $v$ consists of, a representation of the photo of $v$, or the DNA %sequence of $v$. 

Given a root to leaf path $v_0,\ldots,v_{\ell}$, we will consider $R(v_{\ell-1})$ as a higher level representation of $R(v_{\ell})$,. Similarly $R(v_{\ell-2})$ is a higher level representation of 
$R(v_{\ell-1})$ (and therefore of $R(v_{\ell})$. Thus, each higher level representation has many descendant lower level representations. In particular, all of the representations considered are derived from $R(v_0)$. 

\subsection{Labels}

Each node $v \in V$ of the tree has a set of labels $L(v)$. $L(v)$ may be empty for some nodes $v$. 
We require that if $w$ is a descendant of $v$ then $L(v) \subseteq L(w)$ and that every two nodes
$v_1,v_2$ that have the same label $\ell$ have a common ancestor $v_3$ with label $\ell$. In other words, the set of nodes labeled by a certain label is a node in the tree and all nodes below that node. 

For example, a possible value for $L(v)$ is $\{ ``dog", ``german shepherd" \}$.

\subsection{The inference problem}

Let $L_T$ denote the set of leaves of $T$. 
Let $S \subset L_T$. 
The input to the inference problem consists of the set $\{(R(v), L(v)) : v \in S\}$ which is the {\em labeled data} and 
the set $\{R(v) : v \in L_T \setminus S\}$ which is the {\em unlabeled data}. 

The desired output is $L(v)$ for all $v \in L_T$, i.e., the labels of all the leaves of the tree.

\subsection{Generative Models}
We consider a number of increasingly complex generative models. While all of the models are stylized, the more advanced ones capture more of the features of ``deep learning" compared to the simpler models. All of the models will be Markov models on the tree $T$ (rooted at $v_0$). 
In other words, for each directed edge of the tree from a parent $v$ to child $w$, we have a transition matrix $M_{v,w}$ of size $[q]^k \times [q]^k$ that determines the transition probabilities from the representation $R(v)$ to the representation $R(w)$. We consider the following models:

%\begin{enumerate}
%\item 
\subsubsection{The i.i.d. Model (IIDM)} 
We first consider one of the simplest and most classical phylogenetic models given by i.i.d. 
symmetric Markov models. Special cases of this model, for $q=2$ or $q=4$ are some of the most basic phylogenetic evolutionary models. 
These models called the CFN and Jukes-Cantor model respectively~\cite{JukesCantor:69,Neyman:71,Farris:73,Cavender:78}. 
The model is defined as follows: If $w$ is the parent of $v$ then for each $1 \leq i \leq k$ 
independently, it holds that conditioned on $R(w)$ for all $a \in [q]$: 
\[
P[R(v)_i = a] = \frac{1-\lambda}{q} + \lambda \delta(R(w)_i = a). 
\]
In words, for each letter of $R(w)$ independently, the letter given by the parent is copied with probability $\lambda$ and is otherwise chosen uniformly at random. 

\subsubsection{The Varying Representation Model (VRM)}
One of the reasons the model above is simpler than deep learning models is that the representation of nodes is canonical. 
For example, the model above is a classical model if $R(w)$ is the DNA sequence of node $w$ but is a poor model if we consider $R(w)$ to be the image of $w$, where 
we expect different levels of representation to have different ``meanings".  
In order to model the non-canonical nature of neural networks we will modify the above representation as follows. For each edge $e = (w,v)$ directed from the parent $w$ 
to the child $v$,  we associate a permutation $\sigma_e \in S_q$, which encodes the relative 
representation between $w$ and $v$. Now, 
 we still let different letters evolve independently but with different encodings for different edges. So we let: 

 \[
P[R(v)_i = a] = \frac{1-\lambda}{q} + \lambda \delta(R(w)_i = \sigma_e^{-1}(a)). 
\]
In words, each edge of the tree uses a different representation of the set $[q]$. 

We say the the collection $\sigma = (\sigma_e : v \in E)$ is {\em adversarial} if $\sigma$ is chosen by an adversary. We say that it is {\em random} if $\sigma_e$ are chosen i.i.d. uniform. We say that 
$\sigma = (\sigma_e : v \in E)$ are {\em shared parameters} if $\sigma_{e}$ is just a function of the level of the edge $e$. 

\subsubsection{The Feature Interaction Model (FIM)} 
The additional property we would like to introduce in our most complex model is that of an interaction between features. While the second model introduces some indirect interaction between features emanating from the shared representation, deep nets include stronger interaction.  To model the interaction for each directed edge 
$e=(w,v)$, we let $\sigma_e \in S_{q^2}$. We can view $\sigma_e$ as a function from $[q]^2 \to [q]^2$ and it will be useful for us to represent it as a pair of functions $\sigma_e = (f_e,g_e)$ where $f_e : [q]^2 \to [q]$ and $g_e : [q]^2 \to [q]$. We also 
introduce permutations $\Sigma_1,\ldots,\Sigma_h \in S_k$ which correspond to rewiring between the different levels. 
We then let 
\[
P[\widetilde{R(v)}_i = a] = \frac{1-\lambda}{q} + \lambda \delta(R(w)_i = a), 
\]
and 
\[
R(v)_{2i} = f_{e}(\Sigma_{|v|}(\widetilde{R(w)}(2i)), \Sigma_{|v|}(\widetilde{R(w)}(2i+1))
\]
\[
R(v)_{2i+1} = g_{e}(\Sigma_{|v|}(\widetilde{R(w)}(2i)), \Sigma_{|v|}(\widetilde{R(w)}(2i+1)).
\]

In words, two features at the parent mutate to generate two features at the child. The wiring between different features at different levels is given by some known permutation that is level dependent. 
This model resembles many of the convolutional network models and our model and results easily extend to  other variants of interactions between features. 
For technical reasons we will require that for all $i,j$ it holds that 
\begin{equation} \label{eq:rewire_cond}
\{ \Sigma_j(2i),\Sigma_j(2i+1) \} \neq \{2i, 2i + 1 \}. 
\end{equation}
In other words, the permutations actually permute the letters. It is easy to extend the model and our results to models that have more than two features interact. 

\subsection{The parameter sharing setup}
While the traditional view of deep learning is in understanding one object - which is the deep net, our perspective is different as we consider the space of all objects that can be encoded by the network and the relations between them. The two-point of view are consistent in some cases though. 
We say that the VRM model is {\em fixed parametrization} if the permutation $\sigma_e$ are the same for all the edges at the same level. 
Similarly the FIM is {\em parameter shared} if the functions $(f_e,g_e)$ depend on the level of the edge $e$ only. 
For an FIM model with a fixed parametrization, the deep network that is associated with the model is just given by the permutation $\Sigma_i \in S_k$ and 
the permutations $(f^1,g^1),\ldots,(f^h,g^h)$ only. While our results and models are stated more generally, the shared parametrization setup deserves special attention:
\begin{itemize}
\item Algorithmically: The shared parametrization problem is obviously easier - in particular, one expects, that as in practice, after the parameters $\sigma$ and $\Sigma$ are learned, classification tasks per object should be performed very efficiently.
\item Lower bounds: our lower bounds hold also for the shared parametrization setup. However as stated in conjecture~\ref{conj:stronger}, we expect much stronger lower bounds for the FIM model. We expect that such lower bound hold even in the shared parametrization setup.
\end{itemize} 

%In the shared parametrization setup it also makes sense to consider labels $w$ which are defined as simple functions of the representation at some fixed level. 
%In cases, where lower level representations approximately determine higher-level representations, it is therefore, possible to infer the label with small error. 

\section{Shallow, Local and Deep Learning}

We define ``deep learning" indirectly by giving two definitions of ``shallow" learning and of 
``local" learning. Deep learning will be defined implicitly as learning that is neither local nor shallow.  

A key feature that is observed in deep learning is the use of the correlation between features. 
We will call algorithms that do not use this correlation or use the correlation in a limited fashion shallow algorithms. 

Recall that the input to the inference problem $D$ is the union of the labeled and unlabeled data  
\[
D := \{(R(v), L(v)) : v \in S\} \cup \{R(v) : v \in L_T \setminus S\}.  
\]

\begin{definition}
Let $A = (A_1,\ldots,A_j)$ where $A_i \subset [k]$ for $1 \leq i \leq j$. 
The compression of the data according to $A$, denoted $C_A(D)$, 
is 
\[
C_A(D) := \{R(v) : v \in L_T \setminus S\} 
\cup \left( n_D(A_i,x,\ell) : 1 \leq i \leq j, x \in [q]^{|A_i|}, \ell \mbox{ is a label } \right),
\]
where 
for every possible label $\ell$, $1 \leq i \leq j$ and $x \in [q]^{|A_i|}$, 
we define
\[
n_D(A_i,x,\ell) := \# \{ v \in S : L(v) = \ell, R(v)_{A_i} = x\}. 
\]
The canonical compression of the data, $C_{\ast}(D)$, is given by 
$C_{A_1,\ldots,A_k}(D)$, where $A_i = \{ i \}$ for all $i$. 
\end{definition}

In words, the canonical compression gives for every label $\ell$ and every $1 \leq i \leq k$, the histogram of the $i$'th letter (or column) of the representation among all labeled data with label $\ell$. Note that the unlabeled data is still given uncompressed. 

The more general definition of compression allows for histograms of joint distributions of multiple letters (columns). Note in particular that if $A = (A_1)$ and $A_1 = [k]$, then we may identify $C_A(D)$ with $D$ as no compression is taking place. 

\begin{definition}
We say that an inference algorithm is $s$-shallow if the output of the algorithm as a function of the data depends only on $C_A(D)$ where $A = (A_1,\ldots,A_j)$ and each $A_i$ is of size at most $s$. We say that an inference algorithm is shallow if 
the output of the algorithm as a function of the data depends only on $C_{\ast}(D)$. 

In both cases, we allow the algorithm to be randomized, i.e., the output may depend on a source of randomness that is independent of the data.
\end{definition} 

We next define what local learning means. 
\begin{definition}
Given the data $D$, we say that an algorithm is local if for each $R(w)$ for $w \in L_T \setminus S$, 
the label of $v$ is determined only by the representation of $w$, $R(w)$, and all labeled data 
$\{(R(v), L(v)) : v \in S\}$.
\end{definition}

Compared to the definition of shallow learning, here we do not compress the labeled data. 
However, the algorithm has to identify the label of each unlabeled data point without access to 
the rest of the unlabeled data.

\section{Main Results} 
Our main results include positive statements establishing that deep learning labels correctly in some regimes along with negative statements that establish that shallow or local learning do not. 
Combining both the positive and negative results establishes a large domain of the parameter space where deep learning is effective while shallow or local learning isn't.  We conjecture that the lower bounds in our paper can be further improved to yield a much stronger separation (see conjecture~\ref{conj:stronger}). 

\subsection{Main parameters}
The results are stated are in terms of 
\begin{itemize}
\item the branching rate of the tree $d$,
\item the noise level $1-\lambda$,
\item the alphabet size $[q]$ and  
\item The geometry of the set $S$ of labeled data. 
\end{itemize}
Another crucial parameter is $k$, the length of representation. 
We will consider $k$ to be between logarithmic and polynomial in $n = d^h$.  

We will also require the following definition. 
\begin{definition}
Let $\ell$ be a label. We say that $\ell$ is {\em well represented} in $S$ if the following holds.
Let $v \in V$ be the vertex closest to the root $v_0$ that is labeled by $\ell$ (i.e., the set of labels of $v_0$ contains $\ell$). 
Then there are two edge disjoint path from $v$ to $v_1 \in S$ and to $v_2 \in S$.
\end{definition}

The following is immediate
\begin{proposition}
If the tree $T$ is known and if $\ell$ is well represented in $S$ then all leaves whose label is $\ell$ can be identified.
\end{proposition}

\begin{proof} 
In general the set of leaves $L_{\ell}$, labeled by $\ell$
contains the leaves $L'$ of the subtree rooted at the most common ancestor of all the elements of $S$ labeled by $\ell$. If $\ell$ is well represented, then $L_{\ell} = L'$. 
\end{proof}   

\subsection{The power of deep learning}
\begin{theorem} \label{thm:deep} 
Assume that $d \lam^2 > 1$ or $d \lam > 1 + \eps$ and $q \geq q(\eps)$ is sufficiently large. 
Assume further that $k \geq C \log n$. 
Then the following holds for all three models (IIDM, VRM, FIM) with high probability: 
\begin{itemize}
\item The tree $T$ can be reconstructed. 
In other words, there exists an efficient (deep learning) algorithm that 
for any two representation $R(u)$ and $R(v)$, where $u$ and $v$ are leaves of the tree, computes their graph distance.  
\item For all labels $\ell$ that are well represented in $S$, all leaves labeled by $\ell$ can be identified. 
%\item If additionally, each node has at most one label, the number of elements of $S$ with each label is exactly $2$ and the number of leaves in $L_T$ with each label is at least $3$, then a shallow algorithm can identify all nodes correctly. 

\end{itemize} 
\end{theorem}

\subsection{The weakness of shallow and local learning}
%We now state the results which establish the limited power of local and shallow learning. 
%The second part of Theorem~\ref{thm:deep} implies that shallow deep learning is possible even if the amount of labeled data per label is small. 
%We will show that failure of shallow algorithms can be attributed to the marginalization over a large collection of labeled data. 
%In contrast, we expect that if most of the data is labeled, then local algorithms should do well, by using, for example, some type of nearest-neighbor 
%algorithm.  

We consider two families of lower bounds - for local algorithms and for shallow algorithms. 

In both cases,  we prove information theory lower bounds by defining distributions on instances and showing that shallow/local algorithms do not perform well against these  distributions. 

\subsubsection{The Instances and lower bounds} 
Let $h_0 < h_1 < h$. The instance is defined as follows.  Let $\dist$ denote the graph distance 
\begin{definition} \label{def:instance}
The distribution over instances $I(h_0,h_1)$ is defined as follows. Initialize $S = \emptyset$.
\begin{itemize}
\item All nodes $v_1$ with $\dist(v_1,v_0) < h_0$ are not labeled.
\item The nodes with $\dist(v_1,v_0) = h_0$ are labeled by a random permutation of 
$''1",\ldots,''d^{h_0}!"$. 
\item For each node $v_1$ with $\dist(v_1,v_0) = h_0$, pick two random descendants 
$v_2, v_2'$ with $\dist(v_2',v_0) = \dist(v_2,v_0) = h_1$, such that the most common ancestor of $v_2, v_2'$ is $v_1$. 
Add to $S$  all leaves in $L_T$ that are descendants of $v_2$ and $v_2'$. 
\end{itemize}
\end{definition} 

\begin{theorem} \label{thm:local} 
Given an instance drawn from~\ref{def:instance} and data generated from IIDM, VRM or FIM,  
the probability that a local algorithm labels a random leaf in $L_T \setminus S$ correctly is bounded by 
\[
d^{-h_0}(1+ O(k \lambda^{h-h_1} q)).
\]
\end{theorem}
Note that given the distribution specified in the theorem, it is trivial to label a leaf correctly with probability $d^{-h_0}$ by assigning it any fixed label. As expected our bound is weaker for longer representations. 
%An artifact of the method of proof the results in weaker bounds for higher values of $h_1$ is. 
A good choice for $h_1$ is $h_0 + 1$ (or $h_0+2$ if $d=2$), while a good choice of $h_0$ is $1$, where we get the bound
\[
d^{-1}(1+O(k \lambda^h) q),
\]
compared to $d^{-1}$ which can be achieved trivially. 

%\subsubsection{Lower Bounds for Shallow Algorithms} 

\begin{theorem} \label{thm:shallow}
Consider a compression of the data $C_A(D)$, where $A = (A_1,\ldots,A_m)$ and let $s = max_{j \leq m} |A_{j}|$. 
If $d \lam^2 < 1$ and given an instance drawn from~\ref{def:instance} and data generated from IIDM, VRM or FIM, 
the probability that a shallow algorithm labels a random leaf in 
$L_T \setminus S$ correctly is at most 
\[
d^{-h_0} + C m d^{h_0}  \exp(-c (h-h_1)),
\]
where $c$ and $C$ are positive constants which depend on $\lambda$ and $s$.
\end{theorem} 

Again, it is trivial to label nodes correctly with probability $d^{-h_0}$.
For example if $h_0 = 1, h_1 = 2$ (or $=3$ to allow for $d=2$) and we look at the canonical compression $C_*(D)$, we obtain the bound 
$d^{-1} + O(k \exp(-c h)) = d^{-1} + O(k n^{-\alpha})$ for some $\alpha > 0$. Thus when $k$ is logarithmic of polylogarithmic in $n$, it is information theoretically impossible to label 
better than random. 

%We note two important difference between the statements of Theorem~\ref{thm:local} and Theorem~\ref{thm:shallow}.
%\begin{itemize}
%\item Theorem~\ref{thm:shallow} requires that $d \lam^2 < 1$, while in Theorem~\ref{thm:local} this is not needed.
%\item The bound in Theorem~\ref{thm:local} depends on the amount of labeled data, while in Theorem~\ref{thm:shallow} it is independent of the amount of labeled data.
%\end{itemize} 

\section{Proof Ideas}
A key idea in the proof is the fact that Belief Propagation cannot be approximated well by linear functions. Consider the IIDM model with $k=1$ and a known tree. If we wish to estimate the root representation, given the leaf representations, we can easily compute the posterior using Belief Propagation. Belief Propagation is a recursive, thus deep algorithm. Can it be performed by a shallow net? 

While we do not answer this question directly, our results crucially rely on the fact that Belief Propagation is not well approximated by one layer nets. Our proofs build on and strengthen results in the reconstruction on trees community \cite{MosselPeres:03,JansonMossel:04} by showing that there is a regime of parameters where Belief Propagation has a very good probability of estimating the roof from the leaves. Yet, $1$ layer nets have an exponentially small correlation in their estimate. 

The connection between reconstructing the roof value for a given tree and the structural question of reconstructing the root has been studied extensively in the phylogenetic literature since the work 
of ~\cite{Mossel:04a} and the reminder of our proof builds on this connection to establish the main results.

\section{Discussion}
Theorem~\ref{thm:local} establishes that local algorithms are inferior to deep algorithms if most of the data is unlabeled. 
Theorem~\ref{thm:shallow} shows that in the regime where $\lambda^{-1} \in (\sqrt{d},d)$ and for large enough $q$, shallow algorithms are inferior to 
deep algorithms. We conjecture that stronger lower bound and therefore stronger separation can be obtained for the VRM and FIM models. In particular:
\begin{conjecture} \label{conj:stronger}
In the setup of Theorem~\ref{thm:shallow} and the VRM model, the results of the theorem extend to the regime
$d \lam^4 < 1$. In the case of the FIM model it extends to a regime where $\lam < 1-\phi(d,h)$, where $\phi$ decays exponentially in $h$. 
\end{conjecture}

\subsection{Random Trees}
The assumption that the generative trees are regular was made for the ease of expositions and proofs. 
A natural follow up step is to extend our results to the much more realistic setup of randomly generated trees. 

\subsection{Better models}
Other than random trees, better models should include the following:
\begin{itemize}
\item The VRM and FIM both allow for the change of representation and for feature interaction to vary arbitrarily between different edges.
More realistic models should penalize variation in these parameters. This should make the learning task easier.
\item The FIM model allows interaction only between fixed nodes in one level to the next. This is similar to convolutional networks. 
However, for many other applications, it makes sense to allow interaction with a small but varying number of nodes with a preference towards certain localities. 
It is interesting to extend the models and results in such fashion.
\item It is interesting to consider non-tree generating networks. In many applications involving vision and language, it makes sense to allow to ``concatenate" two or more representations. We leave such models for future work. 
\item As mentioned earlier, our work circumvents autoencoders and issues of overfitting by using compact, almost 1-1 dense representations. 
It is interesting to combine our ``global' framework with ``local" autoencoders.
\end{itemize}

\subsection{More Robust Algorithms}
The combinatorial algorithms presented in the paper assume that the data is generated accurately from the model. 
It is interesting to develop a robust algorithm that is effective for data that is approximately generated from the model. 
In particular, it is very interesting to study if the standard optimization algorithms that are used in deep learning are as efficient in recovering the models presented here. We note that for the phylogenetic  reconstruction problem, even showing that the Maximum Likelihood tree is the correct one is a highly non-trivial task and we still do not have a proof that standard algorithms for finding the tree, actually find one, see~\cite{RochSly:15}. 

\subsection{Depth Lower bounds for Belief Propagation}
Our result suggest the following natural open problem:

\begin{problem}
Consider the broadcasting process with $k=1$ and large $q$ in the regime 
$\lambda \in (1/d, 1/\sqrt{d})$. Is it true that the BP function is uncorrelated with any network of size 
polynomial in $d^h$ and depth $o(h)$? 
\end{problem} 

\section{The power of deep learning: proofs}
The proof of all positive results is based on the following strategy:
\begin{itemize}
\item Using the representations $\{ R(v) : v \in L_T\}$ reconstruct the tree $T$.
\item For each label $\ell$, find the most common $w$ ancestor of 
$\{v : v \in L_T, R(v) \in S, L(v) = \ell\}$ and label all nodes in the subtree root at $w$ by 
$\ell$. 
\end{itemize} 

For labels that are well represented, it follows that if the tree constructed at the first step is the indeed the generative tree, then the identification procedure at the second step indeed identifies all labels accurately. 

The reconstruction of the tree $T$ is based on the following simple iterative ``deep" algorithm in which we iterate the following.    
Set $h' = h$.
\begin{enumerate}
\item[LS] This step computes the Local Structure of the tree: 
 For each $w_1,w_2$ with $\dist(v_0,w_1) = \dist(v_0,w_2) = h'$, compute
$\min(\dist(w_1,w_2),2r +2)$. 
This identifies the structure of the tree in levels $\min(h'-r,\ldots,h')$. 
\item[Cond] If $h'-r \leq 0$ then $EXIT$, otherwise, set $h' := h'-r$. 
\item[AR] Ancestral Reconstruction. For each node $w$ with $\dist(v_0,w) = h'$, estimate 
the representation $R(w)$ from all its descendants at level $h'+r$. 
\end{enumerate} 

This meta algorithm follows the main phylogenetic algorithm in~\cite{Mossel:04a}. 
We give more details on the implementation of the algorithm in the $3$ setups.

\subsection{IIDM}
We begin with the easiest setup and explain the necessary modification for the more complicated ones later. 

The analysis will use the following result from the theory of reconstruction on trees. 

\begin{proposition} \label{prop:ancestral} 
Assume that $d \lam^2 > 1$ or $d \lam > 1 + \eps$ and $q \geq q(\eps)$ is sufficiently large. Then there exists $\lam_1 >  0$ and $r$ such that the following holds. 
Consider a variant of the IIDM model with $r$ levels and $k=1$.
For each leaf $v$, let $R'(v) \sim \lam_2 \delta_R(v) + (1-\lam_2) U$, where 
$\lam_2 > \lam_1$ and $U$ is a uniform label. Then there exists an algorithm that given the tree $T$, 
and $(R'(v) :  v \in L_T)$ returns $R'(v_0)$ such that 
$R'(v) \sim \lam_3 \delta_R(v) + (1-\lam_3) U$ where $\lam_3 > \lam_1$. 
\end{proposition} 

\begin{proof}
For the case of $d \lam^2 > 1$ this follows from\cite{KestenStigum:66,MosselPeres:03}. 
In the other case, this follows from~\cite{Mossel:01}.  
\end{proof}

Let $\lam(h')$ denote the quality of the reconstructed representations at level $h'$. 
We will show by induction that $\lam(h') > \lam_1$ and that the distances between nodes 
are estimated accurately. The base case is easy as we can accurately estimate the distance of each node to itself and further 
 take $\lam(h) = 1$. 

To estimate the distance between $w_1$ and $w_2$ we note that the expected normalized hamming distance $d_H(R(w_1),R(w_2))$ between $R(w_1)$ and $R(w_2)$ is:
\[
\frac{q-1}{q} (1-\lam(h')^2 \lam^{\dist(w_1,w_2)})
\]
and moreover the Hamming distance is concentrated around the mean. 
Thus if $k \geq C(\lam',q,r) \log n$ then all distances up to $2 r$ will be estimated accurately and moreover 
all other distances will be classified correctly as being larger than $2r + 2$. 
This establishes that the step LS is accurate with high probability. 
We then apply Proposition~\ref{prop:ancestral} to recover $\hat{R}(v)$ for nodes at level 
$h'$. We conclude that indeed $\lam(h')  > \lam_1$ for the new value of $h'$.  

\subsection{VRM}

The basic algorithm for the VRM model is similar with the following two modifications: 
\begin{itemize}
\item
When estimating graph distance instead of the Hamming distance $d_H(R(w_1),R(w_2))$, we compute 
\begin{equation} \label{eq:rel_h_dist}
d_H'(R(w_1),R(w_2)) = \min_{\sigma \in S_q} d_H \left(\sigma \left( R(w_1) \right),R(w_2) \right),
\end{equation}
i.e., the minimal Hamming distance over all relative representations of $w_1$ and $w_2$. Again, using standard concentration results, we see that if  $k \geq C(\lam',q,r) \log n$ then all distances up to $2 r$ will be estimated accurately. Moreover, for any two nodes 
$w_1, w_2$ of distance at most $2 r$, the minimizer $\sigma$ in (\ref{eq:rel_h_dist}) is unique and equal to the relative permutation with high probability. We write $\sigma(w_1,w_2)$ for the permutation where the minimum is attained. 
\item 
To perform ancestral reconstruction, we apply the same algorithm as before with the following modification: Given a node $v$ at level $h'$ and all of its descendants at level 
$r+h'$, $w_1,\ldots,w_{h^d}$. We apply the reconstruction algorithm in 
Proposition~\ref{prop:ancestral} to the sequences 
\[
R(w_1), \sigma(w_1,w_2)(R(w_2)),\ldots, \sigma(w_1,w_{h^d})(R(w_{h^d})),
\]
where recall that $\sigma(w_1,w_j)$ is the permutation that minimizes the Hamming distance between $R(w_1)$ and $R(w_j)$. 
This will insure that the sequence $\hat{R}(v)$ has the right statistical properties. Note that additionally to the noise in the reconstruction process, it is also permuted by $\sigma(w_1,v)$.  
\end{itemize} 

\subsection{FIM}
The analysis of FIM is similar to VRM. The main difference is that while in VRM model, we reconstructed each sequence up to a permutation $\sigma \in S_q$,  
 in the FIM model there are permutations over $S_{q^2}$ and different permutations do not compose as they apply to different pairs of positions.  
In order to overcome this problem, as we recover the tree structure, we also recover the permutation $(f_e,g_e)$ up to a permutation $\sigma \in S_q$ that is applied to each letter individually. 

For simplicity of the arguments, we assume that $d \geq 3$. 
 Let $w_1,w_2, w_3$ be three vertices that are identified as siblings in the tree and let $w$ be their parent. 
 We know that  $R(w_1), R(w_2)$ and $R(w_3)$ are noisy versions of $R(w)$, composed with  permutations $\tau_1,\tau_2, \tau_3$ on 
$S_{q^2}$.  %As in the VRM model, along with the fact that $w_1$ and $w_2$ are siblings, we also recover the relative permutation $\tau = \tau_1 {\tau_2}^{-1}$. 

We next apply concentration arguments to learn more about $\tau_1,\tau_2,\tau_3$. To do so, recall that $k \geq C(\lam) \log n$. 
Fix $x = (x_1,x_2) \in [q]^2$, and consider all occurrences of $\tau_1(x)$ in 
$R(w_1)$. Such occurrences are correlated with $\tau_2(x)$ in $R(w_2)$ and $\tau_3(x)$ in $R(w_3)$. 
We now consider occurrences of $\tau_2(x)$ in $R(w_2)$ and $\tau_3(x)$ in $R(w_3)$. Again the most common co-occurrence in $w_1$ 
is $\tau_1(x)$.  The following most likely occurrences values will be the $2q-1$ values $y$ obtained as $\tau_1(x_1,z_2))$, or $\tau_1((z_1,x_2))$ where
$z_1 \neq x_1, z_2 \neq x_2$. 

In other words, for each value $\tau_1(x)$ we recover the set 
\[
A(x) = B(x) \cup C(x), 
\]
where 
\[
B(x) =   \{ \tau_1(x_1,z_2) : z_2 \neq x_2 \}, \quad 
C(x) =  \{ \tau_1(z_1,x_2) : z_1 \neq x_1 \}.
\]
Note that if $y^1, y^2 \in B(x)$ then $y^2 \in A(y_1)$ but this is not true if 
$y^1 \in B(x)$ and $y^2 \in C(x)$. We can thus recover for every value 
$\tau_1(x)$ not only the set $A(x)$ but also its partition into $B(x)$ and $C(x)$ (without knowing which one is which). 

Our next goal is to recover 
\[
\{ (x,B(x)) : x \in [q]^2 \}, \quad \{ (x,C(x)) : x \in [q]^2 \}
\]
 up to a possible {\em global} flip of $B$ and $C$. 
In order to do so, note that 
\[
\{ x \} \cup B(x) \cup_{y \in C(x)} B(y) = [q]^2,
\]
and if any of the $B(y)$ is replaced by a $C(y)$, this is no longer true. 
Thus once we have identified $B(y)$ for one $y \in C(x)$, we can identify $B(y)$ for all $y \in C(x)$. 
Repeating this for different values of $x$, recovers the desired $B$ and $C$. 

We next want to refine this information even further. 
WLOG let $x = \tau_1(0,0)$ and let $y = \tau_1(a,0) \in C(x)$. 
and $z = (0,b) \in B(x)$. And note that $C(y) \cap B(z)$ contains a single element, i.e., 
$\tau_1(a,b)$. We have thus recovered $\tau_1$ up to a permutation of $S_{[q]}$ as needed. 

After recovering $\tau_1,\tau_2,\tau_3$ etc. we may recover the ancestral state at their parent $w$ up to the following degrees of freedom (and noise)
\begin{itemize}
\item A permutation of $S_q$ applied to each letter individually.
\item A global flip of the sets $B$ and $C$.
\end{itemize}

In general the second degree of freedom cannot be recovered. However if $v_2$ is a sister of $v_1$ (with the same degrees of freedom), then only the correct choice of the $B/C$ flips will minimize the distance defined by taking a minimum over permutations in $S_{q^2}$. Thus by a standard concentration argument we may again recover the global $B/C$ flip and continue recursively. Note that this argument is using condition~(\ref{eq:rewire_cond}).

\section{The limited power of limited algorithms}

\subsection{The Limited Power of Local Algorithms} 

To prove lower bounds it suffices to prove them for the IIDM model as is a special case of the more general models.  We first prove Theorem~\ref{thm:local}.
\begin{proof}

Let $R(w)$ be an unlabeled leaf representation. Let $M = d^{h_0}$. 
Let $u_1,\ldots,u_M$ denote the nodes level $h_0$ and denote their labels by 
$\ell_1,\ldots,\ell_M$. Let $v_i, v_i'$ denote the nodes below $u_i$ at level $h_1$ with the property that the leaves of the tree rooted at $v_i$ are the elements of $S$ with label $\ell_i$.

Let $u_i$ be the root of the tree $w$ belongs to and let $x_i$ be the lowest intersection 
between the path from $w$ to $u_i$ and the path between $v_i$ and $v_i'$. 
We write $h'$ for $\dist(w,x_i)$. Note that $h' \geq h-h_1$. 
For $j \neq i$ let $x_j$ be the node on the path between $v_j$ and $v_j'$ 
such that $\dist(v_j,x_j) = \dist(v_i,x_i)$. 
We assume that in addition to the labeled data we are also given $h'$ and 
\[
D' = (\ell_1,R(x_1)),\ldots,(\ell_M,R(x_m)).
\] 
Note that we are not given the index $i$.

Of course having more information reduces the probability of error in labeling $R(w)$.
However, note that $R(w)$ is independent of $\{(R(v), L(v)) : v \in S\}$ conditioned on 
$D'$ and $h'$. It thus suffices to upper bound the probability of labeling $R(w)$ correctly 
given $D'$. By Bayes:
\[
P[L(w) = \ell_i | D', h'] 
= \frac{P[R(w) | D', L(w) = \ell_i, h']}{\sum_{j=1}^M P[R(w) | D', L(w) = \ell_j, h']} 
=  \frac{P[R(w) | R(x_i), h']}{\sum_{j=1}^M P[R(w) | R(x_j), h']}
\]  
We note that 
\[
\left( \frac{(1-\lambda^{h'})/q}{(\lambda^{h'} + (1-\lambda^{h'})/q)} \right)^k \leq 
\frac{P[R(w) | R(x_i), h']}{P[R(w) | R(x_j), h']} 
 \leq \left( \frac{(\lambda^{h'} + (1-\lambda^{h'})/q)}{(1-\lambda^{h'})/q} \right)^k
 \]
 So the ratio is 
 \[
 1 + O(k \lambda^{h'} q) =  1 + O(k \lambda^{h-h_1} q)
 \]
 and therefore the probability of correct labeling is bounded by 
 \[
 \frac{1}{M} (1+ O(k \lambda^{h-h_1} q))
 \]
 as needed.

\end{proof} 

\subsection{On count reconstruction} \label{subsec:counts}
We require the following preliminary result in order to bound the power of local algorithms. 
\begin{lemma} \label{lem:counts}
Consider the IIDM with $d \lambda^2 < 1$ and assume that all the data is labeled and that is compressed as
$C_A(D)$, where $A = (A_1)$ and $A_1 = [k]$, i.e, we are given the counts of the data. 
Let $P_{x,h}$ denote the distribution of $C_A(D)$ conditional on $R(v_0) = x$. 
There there exists a distribution $Q = Q^h$ such that for all $x$, it holds 
and 
\begin{equation} \label{eq:count_coupling}
P_{x,h}= (1-\eta) Q + \eta P'_{x,h} \quad \eta \leq C \exp(-c h),
\end{equation}
where $Q$ is independent of $x$ and $c,C$ are two positive constant which depend on $\lambda$ and $k$ (but not on $h$). 
\end{lemma} 
Our proof builds on the a special case of the result for $k=1$, where~\cite{MosselPeres:03} show that the ``count reconstruction problem is not solvable" which implies the existence of $\eta(h)$ which satisfies 
$\eta(h) \to 0$ as $h \to \infty$. The statement above generalizes the result to all $k$. Moreover, we obtain an exponential bound on $\eta$ in terms of $h$. 

\begin{proof}

Assume first that $k=1$. The proof that the threshold for ``count reconstruction is determined by the second eigenvalue"~\cite{MosselPeres:03}  implies the statement of the lemma with a value $\eta = \eta(h)$ which decays to $0$ as $h \to \infty$. 
Our goal in the lemma above is to obtain a more explicit bound showing an exponential decay in $h$. 
Such exponential decay follows from~\cite{JansonMossel:04} for a different problem of {\em robust reconstruction}.
Robust reconstruction is a variation of the reconstruction problem, where the tree structure is known but the value of each leaf is observed with probability 
$\delta > 0$, independently for each leaf.  \cite{JansonMossel:04} proved that if $d \lambda^2 < 1$ and if $\delta(d,\lambda) > 0$ is small enough then the distributions $S_x$ of the 
 the partially observed leaves given $R(v_0) = x$ satisfy
 \begin{equation} \label{eq:count_coupling2}
S_x = (1-\eta) S + \eta S'_X, \quad \eta \leq C \exp(-c h). 
\end{equation}

From the fact that the census reconstruction problem is not solvable~\cite{MosselPeres:03}, it follows that there exists a fixed $h' = h'(\delta)$ such that for the reconstruction problem with $h'$ levels, the distribution of the counts at the leaves can be coupled for all root values except with probability $\delta$. 
We can now  generate $P_{x,h+h'}$ as follows: we first generate $Q_{x,h}$ which is the representations at level $h$.  Then each node at level $h$ is marked as {\em coupled} with probability $1-\delta$ and {\em uncoupled} with probability $\delta$ independently. 
To generate the census $P_{x,h+h'}$ from $Q_{x,h}$,  as follows: for each coupled node at level 
$h$, we generate the census of the leaves below it at level $h + h'$ conditioned on the coupling being successful (note that this census is independent of the representation of the node at level $h$). 
For uncoupled nodes, we generate the census, conditioned on the coupling being unsuccessful. 
From the description 
it is clear that $P_{x,h+h'}$ can be generated from $Q_{x,h}$. 
Since $Q_{x,h}$ has the representation (\ref{eq:count_coupling2}), it now follows that $P_{x,h+h'}$ has the desired representation 
(\ref{eq:count_coupling}) as needed. 

The case of larger $k$ is identical since the chain on $k$ sequences has the same value of $\lambda$. 
Therefore (\ref{eq:count_coupling}) follows from (\ref{eq:count_coupling2}). 

\end{proof}

\subsection{The limited power of shallow algorithms}

We now prove Theorem~\ref{thm:shallow}. 

\begin{proof}
The idea of the proof is to utilize Lemma~\ref{lem:counts} to show that the compressed labeled data is essentially independent of the unlabeled data. 
Write $M = d^{h_0}$. For each permutation $\sigma$ of the $M$ labels $1,\ldots,M$, we write $P_\sigma$ for the induced distribution on the compressed labeled data $C_A(D)$. Our goal is to show we can write 
\begin{equation} \label{eq:couple_perm}
P_{\sigma} = (1-\eta) P + \eta P'_{\sigma}
\end{equation} 
where $\eta$ is small. 
Note that (\ref{eq:couple_perm}) implies that the probability of labeling a unlabeled leaf accurately is at most $M^{-1} + \eta$. 
Indeed we may consider a problem where in addition to the sample from $P_{\sigma}$ we are also told if it is coming from $P$ or from 
$P'_{\sigma}$. If it is coming from $P$, we know it is generated independently of the labels and therefore we cannot predict better than 
random (i.e. $M^{-1}$). 

Let $R(v_1(1)),R(v_2(1)),\ldots,R(v_1(M)),R(v_2(M))$ denote the representations at the roots of the subtrees of the labeled data. 
Let $I$ denote all of the representations and let $P_I$ denote the distribution of $C_A(D)$ conditioned on these representations. By convexity, to prove the desired coupling it suffices to prove  
\[
P_I = (1-\eta) P + \eta P'_{I}
\]
By applying Lemma~\ref{lem:counts} to each of the trees rooted at $v_1(1),v_2(1),\ldots,v_1(M),v_2(M)$ and to each of the sets $A_i$, we obtain the desired results. 

\end{proof}

%\bibliographystyle{amsalpha}

%\bibliography{my,all} 

\begin{thebibliography}{ABGM14}

\bibitem[ABGM14]{ABGM:14}
Sanjeev Arora, Aditya Bhaskara, Rong Ge, and Tengyu Ma, \emph{Provable bounds
  for learning some deep representations}, International Conference on Machine
  Learning, 2014, pp.~584--592.

\bibitem[BM13]{bruna2013invariant}
Joan Bruna and St{\'e}phane Mallat, \emph{Invariant scattering convolution
  networks}, IEEE transactions on pattern analysis and machine intelligence
  \textbf{35} (2013), no.~8, 1872--1886.

\bibitem[Cav78]{Cavender:78}
J.~A. Cavender, \emph{Taxonomy with confidence}, Math. Biosci. \textbf{40}
  (1978), no.~3-4.

\bibitem[CSS16]{CoShSa:16}
Nadav Cohen, Or~Sharir, and Amnon Shashua, \emph{On the expressive power of
  deep learning: A tensor analysis}, 29th Annual Conference on Learning Theory
  (Columbia University, New York, New York, USA) (Vitaly Feldman, Alexander
  Rakhlin, and Ohad Shamir, eds.), Proceedings of Machine Learning Research,
  vol.~49, PMLR, 23--26 Jun 2016, pp.~698--728.

\bibitem[Dar59]{Darwin:59}
Charles Darwin, \emph{On the origin of species}.

\bibitem[ES16]{EldanShamir:16}
Ronen Eldan and Ohad Shamir, \emph{The power of depth for feedforward neural
  networks}, Conference on Learning Theory, 2016, pp.~907--940.

\bibitem[Far73]{Farris:73}
J.~S. Farris, \emph{A probability model for inferring evolutionary trees},
  Syst. Zool. \textbf{22} (1973), no.~4, 250--256.

\bibitem[Fel04]{Felsenstein:04}
J.~Felsenstein, \emph{Inferring phylogenies}, Sinauer, New York, New York,
  2004.

\bibitem[GBC16]{GoBeCo:16}
Ian Goodfellow, Yoshua Bengio, and Aaron Courville, \emph{Deep learning}, MIT
  Press, 2016.

\bibitem[JC69]{JukesCantor:69}
T.~H. Jukes and C.~Cantor, \emph{Mammalian protein metabolism}, Evolution of
  protein molecules (H.~N. Munro, ed.), Academic Press, 1969, pp.~21--132.

\bibitem[JM04]{JansonMossel:04}
S.~Janson and E.~Mossel, \emph{Robust reconstruction on trees is determined by
  the second eigenvalue}, Ann. Probab. \textbf{32} (2004), 2630--2649.

\bibitem[KS66]{KestenStigum:66}
H.~Kesten and B.~P. Stigum, \emph{Additional limit theorems for indecomposable
  multidimensional {G}alton-{W}atson processes}, Ann. Math. Statist.
  \textbf{37} (1966), 1463--1481.

\bibitem[LBH15]{LeBeHi:15}
Yann LeCun, Yoshua Bengio, and Geoffrey Hinton, \emph{Deep learning}, Nature
  \textbf{521} (2015), no.~7553, 436--444.

\bibitem[MLP16]{mhaskar2016learning}
Hrushikesh Mhaskar, Qianli Liao, and Tomaso Poggio, \emph{Learning functions:
  when is deep better than shallow}, arXiv preprint arXiv:1603.00988 (2016).

\bibitem[Mos01]{Mossel:01}
E.~Mossel, \emph{Reconstruction on trees: beating the second eigenvalue}, Ann.
  Appl. Probab. \textbf{11} (2001), no.~1, 285--300.

\bibitem[Mos04]{Mossel:04a}
E.~Mossel, \emph{Phase transitions in phylogeny}, Trans. Amer. Math. Soc.
  \textbf{356} (2004), no.~6, 2379--2404 (electronic).

\bibitem[MP03]{MosselPeres:03}
E.~Mossel and Y.~Peres, \emph{Information flow on trees}, Ann. Appl. Probab.
  \textbf{13} (2003), no.~3, 817--844.

\bibitem[Ney71]{Neyman:71}
J.~Neyman, \emph{Molecular studies of evolution: a source of novel statistical
  problems}, Statistical desicion theory and related topics (S.~S. Gupta and
  J.~Yackel, eds.), 1971, pp.~1--27.

\bibitem[PNB15]{patel2015probabilistic}
Ankit~B Patel, Tan Nguyen, and Richard~G Baraniuk, \emph{A probabilistic theory
  of deep learning}, arXiv preprint arXiv:1504.00641 (2015).

\bibitem[RS15]{RochSly:15}
S.~{Roch} and A.~{Sly}, \emph{{Phase transition in the sample complexity of
  likelihood-based phylogeny inference}}, ArXiv e-prints (2015).

\bibitem[SS03]{SempleSteel:03}
C.~Semple and M.~Steel, \emph{Phylogenetics}, Mathematics and its Applications
  series, vol.~22, Oxford University Press, 2003.

\bibitem[Tel16]{Telgarsky:16}
Matus Telgarsky, \emph{benefits of depth in neural networks}, Conference on
  Learning Theory, 2016, arXiv preprint arXiv:1602.04485, pp.~1517--1539.

\end{thebibliography}

\end{document}